\documentclass[letterpaper]{article}
\usepackage{aaai}
\usepackage{times}
\usepackage{helvet}
\usepackage{courier}
\usepackage{amsthm}
\usepackage{graphicx}
\usepackage{ifthen}

\usepackage{float, times}
\usepackage{amsmath, amsfonts}
\usepackage{url}
\usepackage{algorithm}
\usepackage{algorithmicx}
\usepackage{algpseudocode}
\usepackage{graphicx}
\usepackage{amssymb}
\usepackage{amsthm}
\usepackage{pstricks}
\usepackage{tikz}
\usetikzlibrary{arrows,shapes}
\usepackage[normalem]{ulem} 
\newcommand\hl{\bgroup\markoverwith
  {\textcolor{yellow}{\rule[-.5ex]{2pt}{2.5ex}}}\ULon}
\newcommand\hlr{\bgroup\markoverwith
  {\textcolor{red}{\rule[-.5ex]{2pt}{2.5ex}}}\ULon}
\newcommand\hlg{\bgroup\markoverwith
  {\textcolor{green}{\rule[-.5ex]{2pt}{2.5ex}}}\ULon}

\newtheorem{theorem}{Theorem}
\newtheorem{proposition}[theorem]{Proposition}
\newtheorem{definition}{Definition}
\newtheorem{corollary}[theorem]{Corollary}
\newtheorem{lemma}[theorem]{Lemma}
\newtheorem{example}{Example}

\renewcommand{\Sigma}{\Gamma}

\begin{document}
\title{On the Measure of the Conflicts: A MUS-Decomposition Based Framework}

\author{Said Jabbour$^{1}$, Yue Ma$^{2}$,  Badran Raddaoui$^{1}$, Lakhdar Sa\"is$^{1}$, Yakoub Salhi$^{1}$\\
$^1$CRIL - CNRS, Universit\'e d'Artois, France\\
\texttt{\{jabbour, raddaoui, sais, salhi\}@cril.fr} \\
\\
$^2$ Technische Universit\"at Dresden,  Institut f\"ur theoretische Informatik, Dresden, Germany\\
\texttt{mayue@tcs.inf.tu-dresden.de} 
}

\maketitle

\begin{abstract}

Measuring inconsistency is viewed as an important issue related to handling inconsistencies. Good measures are supposed to  satisfy a set of rational properties.
However,  defining sound properties is sometimes problematic. In this paper, we emphasize one such property, named \textit{Decomposability}, rarely discussed in the literature due to its modeling difficulties.
To this end, we propose an independent decomposition which is more intuitive than existing proposals.
To analyze inconsistency in a more fine-grained way, we introduce a graph representation of a knowledge base and various MUS-decompositions. One particular MUS-decomposition, named {\em distributable MUS-decomposition} leads to an interesting partition of inconsistencies in a knowledge base such that multiple experts can check inconsistencies in parallel, which is impossible under existing measures. Such particular {\em MUS-decomposition} results in an inconsistency measure that satisfies a number of desired properties. Moreover, we give an upper bound  complexity of the measure that can be computed using 0/1 linear programming or Min Cost Satisfiability problems, and conduct preliminary experiments to show its feasibility.
\end{abstract}

\section{Introduction}\label{sec:intro}

Conflicting information is often unavoidable for  large-sized knowledge bases (KBs for short). Thus, analyzing conflicts has 
gained a considerable attention in Artificial Intelligence research \cite{BertossiHS05}. In the same vein, measuring inconsistency has proved useful and attractive in diverse scenarios, including
software specifications \cite{Ana04}, e-commerce protocols \cite{ChenZZ04}, belief merging \cite{QiLB05}, news reports \cite{Hunter06}, integrity constraints \cite{GrantH06}, requirements engineering \cite{Ana04}, databases \cite{MartinezPSSP07,GrantH13}, semantic web \cite{ZhouHQMHQ09}, and network intrusion detection \cite{McAreaveyLMM11}.

Inconsistency measuring is helpful to compare different knowledge bases and to evaluate their quality \cite{Grant78}. 
A number of logic-based inconsistency measures  have been studied, including  the maximal $\eta$-consistency \cite{Knight02}, measures based on variables or via multi-valued models \cite{Grant78,Hunter02,Oller04,Hunter06,GrantH08,MaQXHL10,XiaoLMQ10,MaQH11}, n-consistency and n-probability \cite{DoderRMO10}, minimal inconsistent subsets based inconsistency measures \cite{HunterK08,MuLJ11,MuLJ12,XiaoM12},  Shapley inconsistency value \cite{HunterK06,HunterK10}, and more recently the inconsistency measurement based on minimal proofs \cite{JabbourR13}. 

There are different ways to categorize the proposed measures. One  way is with respect to their dependence on syntax or semantics: Semantic based ones  aim to compute the proportion of the language that is affected by the inconsistency, via for example paraconsistent semantics.
Whilst, syntax based ones are concerned  with the minimal number of formulae that cause inconsistencies, often through minimal inconsistent subsets. Different measures can also be classified by being formula or knowledge base oriented. For example,  the inconsistency measures in \cite{HunterK06,HunterK10} consist in quantifying the contribution of a formula to the inconsistency of a whole knowledge base containing it, while the other mentioned measures aim to quantify the inconsistency degree a the whole knowledge base.
Some basic properties \cite{HunterK10} such as {\it Consistency, Monotony, Free Formula Independence}, are also proposed to evaluate the quality of   inconsistency measures. 

In this paper, we propose a syntax-based framework to measure inconsistencies\footnote{It can be embedded into   Shapley Inconsistency Value to have a formula oriented measure \cite{HunterK06}.}  using a novel methodology allowing to resolve inconsistencies in a parallel way. To this end,   \emph{distributable MUS-decomposition} and \emph{distribution index} of a KB are introduced. Intuitively, a distributable MUS-decomposition gives a reasonable partition of a KB such that  it allows multiple experts to solve  inconsistencies
in parallel; And the distribution index is the maximal components that a KB 
can be partitioned into. This methodology is of great importance in a scenario where the information in a KB is precious, large, and complex such that removing or weakening information requires intensive and time-consuming interactions with human experts. 
Consider $K=\{a_1, \neg a_1, a_1\vee \neg a_2, a_2, \neg a_2, \cdots, a_{n-1}, \neg a_{n-1}, a_{n-1}\vee \neg a_{n}, a_{n}, \neg a_{n}\}$. 
Intuitively,   $K$ contains a large number of inconsistencies. And interestingly,  our approach can recognize $\{a_i, \neg a_i\}$ as $n$ distributable parts of $K$ such that each expert can focus on verifying a single part carefully and independently\footnote{More details are explained later in the paper.}. 
 In contrast, classical  approaches follow the idea of 
resolving inconsistency  as a whole without being able to break a KB into independent pieces. Take,  for example, the classical Hitting Set approach which identifies a minimal set of formulae, e.g. $\{\neg a_i \mid 1\leq i\leq n\}$ of $K$, to  remove for restoring consistency. Note that $K$ has many such Hitting Sets of  a big size $n$. Therefore, even if working in parallel, each expert needs to verify a large number of formulae, which is time consuming. More problematic in general,  there are often overlaps among  Hitting sets so that multiple experts have to waste time in unnecessarily rechecking the overlaps. This is the same if we simply distribute one minimal inconsistent subsets to an expert. However, the proposed distributable MUS-decomposition avoids  this problem because it gives a disjoint decomposition of a KB.
The methodology is inspired and a side-product of our exploration of the \emph{decomposition} property defined for inconsistency measures, which is rarely discussed in the literature due to its modeling difficulty \cite{HunterK10}.

Our technical contributions are as follows:
\begin{itemize}
	\item  We propose  \emph{independent decomposability} as a more reasonable characterization of inconsistency measures.
	\item We define a \emph{graph representation} of KBs to analyze connections between minimal inconsistent subsets by exploiting the structure of the graph.  Such a representation is then used to improve an existing inconsistency measure to satisfy the independent decomposability.
	\item Based on the graph representation, a series of \emph{MUS-decompositions} are introduced and used for defining the \emph{distribution-based inconsistency measure $I_{\cal D}$}.
	We show the interesting properties of $I_{\cal D}$ and give a comparison with other measures, which indicates its rationality. 
	\item We study the complexity of $I_{\cal D}$ (via an  \emph{extended set packing problem}) and we provide encodings as a 0/1 linear program or min cost satisfiability for its computation. 
\end{itemize}
 
The paper is organized as follows: Sections \ref{sec:preliminaries}  and \ref{sec:MIV} give basis notions and recall some inconsistency measures  relevant to the present work. In Section \ref{sec:partition}, we propose a graph representation of a KB and use it to revise an existing measure. 
 Section \ref{sec:KBmeasure} focuses on \emph{ MUS-decomposition} and \emph{distribution-based inconsistency measure}.  Section \ref{sec:computation} gives the complexity results of the proposed measure and its computation algorithms whose efficiency is evaluated in Section \ref{sec:experiments}. 
Section \ref{sec:conclusion}  concludes the paper with some perspectives.


\section{Preliminaries}\label{sec:preliminaries}
Through this paper, we consider the propositional language $\mathcal{L}$ built over a finite set of propositional symbols $\mathcal{P}$ using 
classical logical connectives $\{ \neg, \wedge, \vee, \rightarrow\}$.
We will use letters such as $a$ and $b$ to denote propositional variables, Greek letters like $\alpha$ and $\beta$ to denote propositional 
formulae. The symbols $\top$ and $\bot$ denote tautology and contradiction, respectively. 

A knowledge base $K$ consists of a finite set of propositional formulae. Sometimes, a propositional formula can be in conjunctive normal form (CNF) i.e. a conjunction of clauses. Where a clause is a disjunction literals, and a literal is either a propositional variable ($x$) or its negation ($\neg x$). 
For a set $S$, $\vert S \vert$ denotes its cardinality. 
Moreover, a KB $K$ is inconsistent if there is a formula $\alpha$ such that $K \vdash \alpha$ and  $K \vdash \neg\alpha$, where $\vdash$ is the deduction in classical propositional logic.
If $K$ is inconsistent,  \textit{Minimal Unsatisfiable Subsets (MUS)} of $K$ are defined as follows:

\begin{definition}[MUS]
Let $K$ be a KB and $M\subseteq K$.
$M$ is a minimal unsatisfiable (inconsistent) subset (MUS) of $K$ iff $M \vdash\bot$ and  $\forall M' \subsetneq M$, $M'\nvdash\bot$. 
The set of all minimal unsatisfiable subsets of $K$ is denoted
$\mathit{MUSes(K)}$. 
\end{definition}

Clearly, an inconsistent KB $K$ can have multiple minimal inconsistent subsets.
When a $\mathit{MUS}$ is singleton,  the single formula in it, is called a \textit{self-contradictory formula}.
We denote the set of self-contradictory formulae of $K$ by $selfC(K) = \{\alpha\in K ~ \vert ~ \{\alpha\}\vdash\bot \}$.  A formula $\alpha$ that is not involved in any MUS of $K$ is called \textit{free formula}.
The set of free formulae of $K$ is written 
$free(K)=\{\alpha \mid  \mbox{ there is no } M\in MUSes(K) \mbox{ such that }\alpha\in M\}$, and  its complement is named \textit{unfree formulae} set, defined as $unfree(K) = K\setminus free(K)$.
Moreover,  the \textit{Maximal Consistent Subset} and  \textit{Hitting set}  are defined as follows: 

\begin{definition}[MSS]
Let $K$ be a KB and $M$ be a subset of $K$.
$M$ is a maximal satisfiable (consistent) subset (MSS) of $K$ iff $M \nvdash\bot$ and $\forall \alpha \in  K\setminus M$, $M \cup\{\alpha\} \vdash\bot$. The set  of all maximal satisfiable subsets is denoted
$\mathit{MSSes(K)}$.

\end{definition}

\begin{definition}\label{def3}
Given a universe $U$ of elements and a collection ${\cal S}$ of subsets of $U$, $H\subseteq U$ is a hitting set of ${\cal S}$ if $\forall E \in {\cal S}, H \cap E \neq \emptyset$. 
 $H$ is a minimal hitting set of ${\cal S}$ if $H$ is a hitting set of ${\cal S}$ and each $H'\subset H$ is not a hitting set of ${\cal S}$. 
\end{definition}

\section{Inconsistency Measures}\label{sec:MIV}

We review the inconsistency measures relevant to the ones proposed in this paper.

There have been several contributions for measuring inconsistency in knowledge bases defined through minimal inconsistent subsets theories.
In \cite{HunterK10}, Hunter and Konieczny introduce a scoring function allowing to measure the degree of inconsistency of a subset of formulae of a given knowledge base.
In other words, for a subset $K' \subseteq K$, the scoring function is defined as the reduction of the number of minimal inconsistent subsets obtained by removing $K'$ from $K$ (i.e. $\vert \mathit{MUSes(K)} \vert - \vert \mathit{MUSes(K-K')}\vert$). 
By extending the scoring function, the authors introduce an inconsistency measure of the whole base, defined as the number of minimal inconsistent subsets of $K$.
Formally, $I_{MI}(K) = \vert \mathit{MUSes(K)} \vert.$\\
$I_{MI}$ measure also leads to an interesting Shapley Inconsistency Value $S_\alpha^{I_{MI}}$  with desirable properties \cite{HunterK10}.  

Combining both  minimal inconsistent subsets and  maximal consistent subsets is another way to define inconsistency degree \cite{MuLJB11,GrantH11}.
We consider the inconsistency value $I_M(K)$ that 
counts for a given KB, the number of its $MSSes$ and its Self-contradictory formulae (subtraction of 1 is required to make $I_M(K)=0$ when $K$ is consistent):
	$$I_M(K) = |MSSes(K) | + |selfC(K)| -1.$$
Another inconsistency measure considered in this paper is defined as the minimum hitting set of $MUSes(K)$: $$\delta_{hs}(K) = min \{|H|~| H \mbox{ is a hitting set of }MUSes(K)\}.$$
$\delta_{hs}(K)$ is the size of the smallest hitting set of $MUSes(K)$ w.r.t. its cardinality. 

In addition, a set of properties  have been proposed to characterize an inconsistency measure. 

\begin{definition}[\cite{HunterK10}]
\label{BIM} Given two knowledge bases $K$ and $K'$, and formulae $\alpha$ and $\beta$ in ${\cal L}$,
\begin{itemize}
	\item[(1)] Consistency: $I(K) = 0$ iff $K$ is consistent
	\item[(2)] Monotony: $I(K)  \le I(K \cup K')  $
	\item[(3)] Free Formula Independence: if $\alpha$ is a free formula in $K \cup \{\alpha \}$, then $I(K \cup \{\alpha\}) = I(K) $
	\item[(4)]  MinInc: If  $M\in MUSes(K),$ then $I(M)=1$.
\end{itemize}
\end{definition}
The monotony property shows that the inconsistency value of a KB increases with the addition of new formulae.
The free formula independence property  states that the set of formulae not involved in any minimal inconsistent subset does not influence the inconsistency measure. 
The MinInc is used to characterize the Shapley Inconsistency Value by $I_{MI}(K)$ in \cite{HunterK08}.

\section{Independent Decomposability Property}\label{sec:partition}
 
There are common properties that we examine for an inconsistency measure  (Definition \ref{BIM}), while leaving another property, called \textit{Decomposability} or \textit{Additivity}, debatable  due to its modelling difficulty ~\cite{HunterK08}. Indeed, properties in Definition \ref{BIM}  have an inspiring root from the axioms of Shapley Value \cite{Shap53}.  As mentioned in \cite{LR57}, one of the main limitation of the original additivity lies in the fact that the interactions of sub-games are not considered. Moreover, \cite{HK06} argue that a direct translation of Shapley's additivity has little sense for inconsistency measures.
For this reason, \textit{Pre-Decomposability} and Decomposability are defined \cite{HunterK10} for formula-oriented inconsistency measures.  

In this section, we analyze the limitation of existing decomposability property and propose an \textit{Independent Decomposability} which is  more intuitive. We then derive a new measure  $I'_M$  by modifying $I_M$ to satisfy the independent decomposability property by considering the  interactions between MUSes through \emph{MUS-graph representation} of a KB.

Let us recall Pre-decomposability and Decomposability properties \cite{HunterK10}.  
 
\begin{definition}[Pre-Decomposability\footnote{It is named  MinInc Separability in \cite{HunterK08}}]\label{def:additivity} 
Let $K_1,\ldots,K_n$ be knowledge bases and
$I$ an inconsistency measure. 
$I$ satisfies Pre-Decomposability if it satisfies the following condition:
If $\mathit{MUSes(K_1\cup\ldots\cup K_n)} = \mathit{MUSes(K_1)} \oplus \ldots \oplus$
$\mathit{MUSes(K_n)}$\footnote{We denote a partition $\{A,B\}$ of a set $C$ by $C = A \oplus B$, i.e., $C = A \cup B$ and $A\cap B = \emptyset.$}, 
then $I(K_1\cup\ldots\cup K_n) = I(K_1) + \ldots + I(K_n)$.
\end{definition}

Pre-Decomposability ensures that the inconsistency degree of a KB $K$ can be obtained by summing up the degrees of its sub-bases $K_i$ under the condition that $\{MUSes(K_i)\mid 1\leq i\leq n\}$ is a partition of $MUSes(K)$. 

\begin{definition}[Decomposability]\label{def:additivity}
$I$ satisfies Decomposability if it satisfies the following condition:
If $|\mathit{MUSes(K_1\cup\ldots\cup K_n)}| = \sum_{1\leq i \leq n}|\mathit{MUSes(K_i)}|$, then
$I(K_1\cup\ldots\cup K_n) = I(K_1) + \ldots + I(K_n)$.
\end{definition}

Compared to Pre-Decomposability, Decomposability characterizes a weaker condition that consider only MUSes cardinalities of $K$ and $K_i$. 
Although Pre-Decomposability and Decomposability can characterize some kind of interactions. We argue that this condition is not sufficient. Let us consider the following example:

\begin{example}\label{exam:enh-additivity}
Let $K_1=\{a,\neg a\}, K_2=\{\neg a,a\wedge b\}, K_3=\{c,\neg c\}$, each of which  contains only one single MUS. Consider two bases 
$K=K_1 \cup K_2, K'= K_1 \cup K_3$.
Clearly, $MUSes(K)$= $MUSes(K_1)$ $\oplus$ $ MUSes(K_2)$, and $MUSes(K')=MUSes(K_1)\oplus MUSes(K_3)$.
For any measure $I$, if $I$ satisfies the decomposability property (Definition \ref{def:additivity}),  we have $I(K)=I(K_1)+I(K_2)$ and $I(K')=I(K_1)+I(K_3)$. Moreover, if $I$ satisfies the MinInc property. 
Then,  K and K' will have the same value, which is counter-intuitive because the components of $MUSes(K')= \{\{a, \neg a\}, \{c, \neg c\}\}$ are unrelated, whereas those of $MUSes(K)=\{\{a, \neg a\}, \{\neg a, a\wedge\neg b \}\}$ are overlapping. Consequently, the components of $MUSes(K')$ are more spread than those of  $MUSes(K)$.
One can expect that $K'$ should contain more inconsistencies than $K$.
\end{example}

This example illustrates the necessity to characterize the interactions among sub-bases whose inconsistency measures can be summed up.
To this end, we propose the following independent decomposability property:

\begin{definition}[Independent  Decomposability]\label{def:enh-additivity}
Let $K_1,$ $\ldots,$ $K_n$ be knowledge bases and $I$ an inconsistency measure. 
If $\mathit{MUSes(K_1\cup\ldots\cup K_n)} = \mathit{MUSes(K_1)} \oplus \ldots \oplus$
$\mathit{MUSes(K_n)}$ and $unfree(K_i) \cap unfree(K_j)=\emptyset $ for all $1\leq i\not=j\leq n$, then 
$I(K_1\cup\ldots\cup K_n) = I(K_1) + \ldots + I(K_n)$.
$I$ is then called ind-decomposable.
\end{definition}

To perform additivity for a given measure, the independent decomposability requires an additional precondition expressing that pairwise sub-bases should not share unfree formulae, which encodes a stronger independence among sub-bases.
Indeed, the independent decomposability avoids the counter-intuitive conclusion illustrated in Example \ref{exam:enh-additivity}.
To illustrate this,  suppose that $I$ satisfies independent decomposability, then we have $I(K')=I(K_1)+I(K_3)$, but not necessarily  $I(K')=I(K_1)+I(K_2)$ as $\mathit{MUSes(K_1)}$ and $\mathit{MUSes(K_2)}$ share the formula $\neg a$.
Hence $I(K)$ can be different from $I(K')$. 

Clearly, the following relations hold among different decomposability conditions.

\begin{proposition}
Decomposability implies Pre-Decomposability; Pre-Decomposability implies Independent Decomposability. 
\end{proposition}
Indeed, as shown by Example \ref{exam:enh-additivity}, the strong constraints of Pre-Decomposability and Decomposability would make an inconsistency measure behavior counter-intuitive. 
In contrast, the independence between sub-bases required in the independent decomposability property make it more intuitive. 

While we can see that the measure $I_{MI}$  is pre-decomposable, decomposable, and ind-decomposable, it is not the case for $I_M$ measure as shown below.

\begin{proposition}\label{prop:decom-relation}
The measure $I_M$ is not pre-decomposable, neither decomposable and nor ind-decomposable.
\end{proposition}

\begin{proof} 
Consider the counter example: $K_1 = \{a, \neg a\}$, $K_2 = \{b, \neg b\}$ and $K = K_1\cup K_2$.
It is easy to check that $K$ and $K_i ~(i=1,2)$ satisfy the conditions of Pre-Decomposability, Decomposability, and Independent Decomposability.
We have $I_M(K_1 \cup K_2) = 3$ while $I_M(K_1) + I_M(K_2) = 2$.
Consequently, $I_M(K_1)+I_M(K_2) \neq I_M(K_1 \cup K_2)$.
Thus, $I_M$ is not pre-decomposable, neither decomposable and nor ind-decomposable.
\end{proof}

Indeed, the following theorem states that under certain constraints, $MSS$ is multiplicative instead of additive.

\begin{theorem}
Let $K=K_1\cup\cdots{}\cup K_n$ be KBs such that $MUSes(K_1\cup\cdots{}\cup K_n) = MUSes(K_1)\oplus\cdots{}\oplus MUSes(K_n)$ and, 
for all $1 \leq i,j\leq n$ with $i\neq j$, $K_i\cap K_j=\emptyset$.
Then, $M\in MSSes(K)$ iff $M=M_1\cup\cdots{}\cup M_n$ where $M_1\in MSSes(K_1),\ldots{}, M_n\in MSSes(K_n)$.
\end{theorem}

\begin{proof}
By induction on $n$.
The case of $n=1$ is trivial.
We now consider the case of $n> 1$. 
Let $K'=K_1\cup \cdots{}\cup K_{n-1}$. Using induction hypothesis, we have $M'\in MSSes(K')$ iff $M=M_1\cup\cdots{}\cup M_{n-1}$ where $M_1\in MSSes(K_1),\ldots{}, M_n\in MSSes(K_{n-1})$.\\
{\it Part $\Rightarrow$}.
Let $M\in MSSes(K'\cup K_n)$.
Then, there exist $M'\subseteq K'$ and $M_n\subseteq K_n$ such that $M=M'\cup M_n$.
If $M'\notin MSSes(K')$ (resp. $M_n\notin MSSes(K_n)$) then there exists $\alpha\in (K'\cup K_n)\setminus M$ such that 
$M'\cup\{\alpha\}$ (resp. $M_n\cup\{\alpha\}$) is consistent.
Using $MUSes(K'\cup K_n)= MUSes(K')\oplus MUSes(K_n)$
and $K'\cap K_n=\emptyset$, $M\cup \{\alpha\}$ is consistent and we get a contradiction.
Therefore, $M'\in MSSes(K')$ and $M_n\in MSSes(K_n)$.\\
{\it Part $\Leftarrow$}.  Let $M'\in MSSes(K')$ and $M_n\in MSSes(K_n)$.
Then, the set $M=M'\cup M_n$ is consistent, 
since we have $M'\cap M_n =\emptyset$ and $MUSes (K'\cup K_n)= MUSes(K')\oplus MUSes(K_n)$.
Let us now show that $M$ is in $MSSes(K'\cup K_n)$.
Assume that $M$ is not in  $MSSes(K'\cup K_n)$. 
Then, there exists $\alpha\in (K'\cup K_n)\setminus M$ such that $M\cup \{\alpha\}$ is consistent. 
If $\alpha\in K'$ (resp. $\alpha\in K_n$), then $M'\cup \{\alpha\}$ (resp. $M_n\cup\{\alpha\}$) is consistent and we get a contradiction. 
Therefore, $M$ is in $MSSes(K'\cup K_n)$.
\end{proof}

Using this theorem, we deduce the following corollary:

\begin{corollary}
\label{prop:indmcs}
Let $K=K_1\cup\cdots{}\cup K_n$ be KBs such that $MUSes(K_1\cup\cdots{}\cup K_n) = MUSes(K_1)\oplus\cdots{}\oplus MUSes(K_n)$ and, 
for all $1\leq i,j\leq n$ with $i\neq j$, $K_i\cap K_j=\emptyset$.
Then, $|MSSes(K)|= |MSSes(K_1)|\times\cdots{}\times |MSSes(K_n)|$.
\end{corollary}

As the Independent Decomposability gives a more intuitive characterization of the interaction among subsets, in the following, we are interested in  restoring  the independent decomposability property of the $I_M$ measure. 

Let us first define two fundamental concepts: \textit{MUS-graph} and \textit{MUS-decomposition}.
 
\begin{definition}[MUS-graph]
\label{def:gmus}
The MUS-graph of $K$ of a KB $K$, denoted ${\cal G}_{{\cal MUS}}(K)$, is an undirected graph where:
\begin{itemize}
\item  $MUSes (K)$ is the set of vertices; and  
\item $\forall M, M' \in \mathit{MUSes(K)}$, $\{M,M'\}$ is an edge iff  $M \cap  M' \neq \emptyset$.  
\end{itemize}
\end{definition}

A MUS-graph of $K$ gives us a structural representation of the connection between minimal unsatisfiable subsets. 

\begin{example}
\label{ex:graph}
Let $K = \{a \wedge d, \neg a, \neg b, b \vee \neg c, \neg c \wedge d, \neg c \vee e, c, \neg e, e \wedge d\}$.
We have $MUSes(K) =\{M_1,\dots, M_5\}$ where 
$M_1 = \{\neg a , a \wedge d\}$,
$M_2 = \{c, \neg b, b \vee \neg c\}$,
$M_3 = \{c, \neg c\wedge d\}$,  $M_4 = \{\neg c \vee e, c, \neg e\}$, and
$M_5 = \{ \neg e, e \wedge d\}$. 
So ${\cal G}_{{\cal MUS}}(K)$ is as follows:

\begin{figure}[!h]
\begin{center}
\includegraphics[width=4cm]{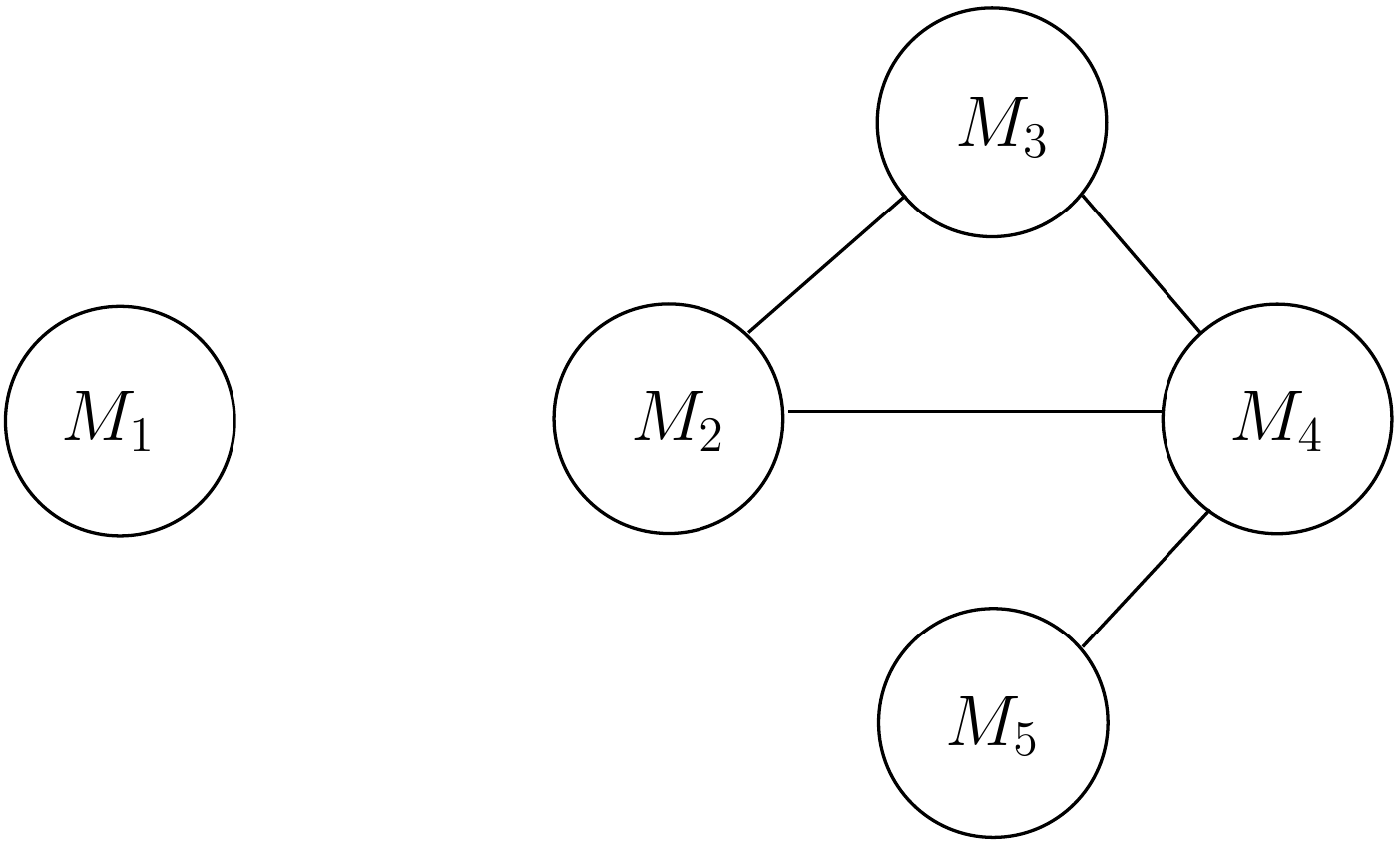} 
\caption{${\cal G}_{{\cal MUS}}(K)$: MUS-graph  of $K$}
\label{fig:musgr}
\end{center}
\end{figure}

\end{example}

Moreover,  ${\cal G}_{{\cal MUS}}(K)$ leads to  a partition 
of a KB $K$, named \textit{MUS-decomposition},  as defined below. 

\begin{definition}[MUS-decomposition]
\label{def:musdec}
A MUS-decomposition of $K$ is a set $\{K_1,\ldots{},K_p\}$ such that 
$K=K_1\cup\dots\cup K_p\cup free(K)$ and  $MUSes(K_i)~(1\leq i\leq p)$ are the connected components of ${\cal G}_{{\cal MUS}}(K)$.


\end{definition}

By the fact that $MUSes(K)\not=\emptyset$ and the uniqueness of the connected components of a graph, we can easily see:

\begin{proposition}
MUS-decomposition exists and is unique for an inconsistent KB.
\end{proposition}


\begin{example}(Example \ref{ex:graph} contd.)
The MUS-decomposition of $K$ contains two components of   ${\cal G}_{{\cal MUS}}(K)$: $K_1 = M_1$ and
$K_2 = M_2\cup M_3\cup M_4\cup M_5$ by noting that $free(K)=\emptyset$. 

\end{example}

Obviously, the MUS-decomposition of a KB can be computed in polynomial time given its MUS-graph.
Interestingly, we can see that the partition $\{K_1,\dots, K_p, free(K)\}$ satisfies the application conditions of Independent Decomposability. 
That is, if an inconsistency measure $I$ is ind-decomposable and free-formula independent,
then $I(K)=I(K_1)+\dots+I(K_p)$.
%

In the following, based on MUS-decomposition, we present an alternative to the inconsistency measure $I_M$ (defined in Section \ref{sec:preliminaries}) so as to make it ind-decomposable. 



\begin{definition}
Let $K$ be a KB with its MUS-decomposition  $K=\{K_1,\dots,K_p\}$.
The $I'_M$  measure is defined as follows:
$$
I'_M(K) = \left\{
    \begin{array}{ll}
      \hspace{-1mm} \sum\limits_{1\le i\le p} |MSSes(K_i)| + |selfC(K)| & \mbox{if } K \vdash\bot; \\
        \\[-6pt]
        \mathit{0 }& \mbox{otherwise}.
    \end{array}
\right.
$$
\end{definition}

That is, instead of $MSSes(K)$ as in $I_M$, the maximal consistent subsets of MUS-decomposition of $K_i$ are used in $I_M'$. 

\begin{example}(Example \ref{ex:graph} contd.)
\label{ex:partitionmus} We have $MSSes(K_1) = \{ 
\{a \wedge d\}, \{ \neg a\}\}$ and $MSSes(K_2) = 
\{\{ \neg b, b \vee \neg c, \neg c \wedge d, \neg c \vee e, e \wedge d\},
\{ \neg b, b \vee \neg c, \neg c \wedge d, \neg c \vee e,  \neg e\},
\{b \vee \neg c, \neg c \vee e, c, e \wedge d\},
\{ \neg b, \neg c \vee e, c, e \wedge d\},
\{b \vee \neg c, c, \neg e\},
\{ \neg b, c, \neg e\}
\}$. Then $I'_M(K) = 2+6 = 8$. 
\end{example}

\begin{proposition}\label{prop2}
$I'_M$  measure is  ind-decomposable.
\end{proposition}

\begin{proof}

Let $K=\bigcup\limits_{1\leq i\leq n} K_i$ be a KB such that $MUSes(K)= \oplus_{1\leq i\leq n} MUSes(K_i)$ and, for all $1\leq i, j\leq n$ with $i\neq j$, $unfree(K_i)\cap~unfree(K_j)=\emptyset$.
One can easily see that $I'_M(K)=0$ if and only if, for all $1\leq i\leq n$, $I'_M(K_i)=0$.
We now consider the case of $I'_M(K)> 0$.
We denote by ${\cal C}(K_i)$ the set of connected components in ${\cal G}_{{\cal MUS}}(K_i)$ for $i=1,\ldots{},n$.
Thus, $\bigcup\limits_{1\leq i\leq n} {\cal C}(K_i)$ is the set of connected components in $ {\cal G}_{{\cal MUS}}(K)$, 
since $ {\cal G}_{{\cal MUS}}(K)=\biguplus\limits_{1\leq i\leq n}{\cal G}_{{\cal MUS}}(K_i)$\footnote{The union of  disjoint graphs is the graph with the union of vertex and edge sets from  individual graphs as its vertexes and edges.}.  
Moreover, it is obvious that  $selfC(K)=\bigcup\limits_{1\leq i\leq n} selfC(K_i)$.
Let $\{K_i^1, \ldots{}, K_i^{p_i}, free (K_i)\}$ be the MUS-decomposition of $K_i$ for $i=1,\ldots{}, n$.
We have $I'_M(K)= \sum\limits_{1\le i\le n}(\sum\limits_{1\le j\le p_i} |MSSes(K^{j}_{i})|+|selfC(K_i)|)= \sum\limits_{1\le i\le n} I'_M(K_i)$, since 
$(\bigcup\limits_{1\leq i\leq n}\{ K_i^1, \ldots{}, K_i^{p_i}\})\cup(\bigcup\limits_{1\leq i\leq n} free(K_i))$ is the MUS-decomposition of $K$. 

\end{proof}

That is, by taking into account the connections between minimal inconsistent subsets,  MUS-decomposition gives us a way to define an inconsistency measure which still satisfies the Independent Decomposability.

%
%
\section{A New MUS-based Inconsistency Measure}\label{sec:KBmeasure}

Recall that we want to have a way to resolve inconsistencies in a parallel way as mentioned in Section \ref{sec:intro}.
Indeed, MUS-decomposition defines a disjoint partitions of a KB.
However, it is inadequate for this purpose.
Consider again $K=\{a_1, \neg a_1, a_1\vee \neg a_2, a_2, \neg a_2, \cdots, a_{n-1}, \neg a_{n-1}, a_{n-1}\vee \neg a_{n}, a_{n}, \neg a_{n}\}$.
The MUS-decomposition can not divide $K$ into smaller pieces because its MUS-graph contains only one connected component.
A solution to this problem is via a more fine-grained analysis of a MUS-graph by taking into account its inner structures.
To this end, we propose \textit{partial} and \textit{distributable} MUS-decompositions, based on which a new inconsistency measure is proposed and shown having more interesting properties.

Let us first study  a general characterization of inconsistency measures with respect to the Independent Decomposability property.

\begin{definition}\label{def:dimk}
Let $K$ be a KB, $\{ K_1, \ldots, K_p\}$ the MUS-decomposition of $K$ and 
$\delta$ a function from $\{K_1, \ldots, K_p\}$ to $\mathbb{R}$.
The MUS-decomposition based inconsistency measure of $K$ with respect to  $\delta$, denoted $I_{{\cal D}}^\delta(K)$, is defined as follows:
$$I_{{\cal D}}^\delta(K) = \displaystyle\sum^{p}_{i=1} ~ {\delta}(K_i)$$
\end{definition}

A range of possible measures can be defined using the above general definition.
Let us review some existing instances of  $I_{{\cal D}}^\delta$  according to some  $\delta$ functions.
The simplest one is obtained when ${\delta}(K_i)=1$.
In this case, we get a measure  that assigns to $K$ the number of its connected components.
However, this measure in not monotonic.
Indeed, adding new formulae to a KB can decrease the number of connected components.
For instance, consider the KB $K = \{a, \neg a, b, \neg b\}$ that contains two singleton connected components 
$K_1 = \{\{a, \neg a\}\}$ and $K_2 = \{\{b, \neg b\}\}$. Now, adding the formula $a \vee b$ to $K$ leads to a new KB containing a unique connected component
$K = \{\{a, \neg a\}, \{b, \neg b\}, \{\neg a, a \vee b, \neg b\}\}$. 
Besides, this simple measure considers each connected component as an inseparable entity.

Moreover, when we take $\delta(K_i) = |K_i|$  
(the number of $MUSes$  involved in the connected component $K_i$), $I_{{\cal D}}^\delta(K)$ is equal to $I_{MI}$ measure i.e. $I_{{\cal D}}^\delta(K) =  |\mathit{MUSes(K)}|$. 
This measure again does not take into account the inner structure of minimal inconsistent subsets of a $K$. 

\subsection{(Maximal) Partial MUS-decomposition}

We now modify $I_{{\cal D}}^\delta$ to take into account interactions between MUSes. 
In particular, we  deeply explore the Independent Decomposability and the Monotony properties to define a new inconsistency measure, while keeping other desired properties satisfied.
To this end, we  first introduce the partial MUS-decomposition  notion.

\begin{definition}[Partial MUS-decomposition]\label{def:ccpartition}
Let $K$ be a KB and $K_1,\ldots{},K_n$ subsets of $K$. 
The set $\{K_1,\ldots, K_n\}$ is called a partial  MUS-decomposition of $K$  if the following conditions are satisfied:
	\begin{itemize}
		\item[(1)] $K_i\vdash\bot,\ \ $  for  $1\leq i\leq n$;
		\item[(2)] $\mathit{MUSes(K_1 \cup\ldots\cup K_n)} = \bigoplus_{1\leq i\leq n} 	\mathit{MUSes(K_i)}$;
		\item[(3)] $K_i\cap K_j = \emptyset$, $\forall ~i\neq j$.
	\end{itemize}
	We denote  $pMUSd(K)$ the  set of partial MUS-decompositions of $K$.
\end{definition}

The following proposition comes from the fact that the MUS-decomposition of a KB $K$ is 
in $pMUSd(K)$.

\begin{proposition}
Any inconsistent KB  has at least one partial MUS-decomposition. 
\end{proposition}
Unlike the uniqueness of MUS-decomposition, a KB can have multiple partial MUS-decompositions as shown in the following example.

\begin{example}
\label{ex:icc3}
Consider $K = \{a, \neg a,  a \vee b, \neg b, b, c, \neg c \wedge d, \neg d \wedge e \wedge f, \neg e, \neg f\}$. 
Figure \ref{fig:ubimuslb} depicts the graph representation of $K$
 which contains two connected components $\mathcal{C}_1$ and $\mathcal{C}_2$ where $\mathcal{C}_1 = \{a, \neg a,  a \vee b, \neg b, b\}$ and $\mathcal{C}_2 = \{c, \neg c \wedge d, \neg d \wedge e \wedge f, \neg e, \neg f\}$.
 So the MUS-decomposition of $K$ is $\{\mathcal{C}_1, \mathcal{C}_2\}$. However, there are many partial MUS-decompositions with some examples listed below:
\begin{itemize}
	\item $K_1=\{a, \neg a\}$, and $K_2=\{b, \neg b\}$.
	\item $K_1'=\{a, \neg a\}$, $K_2'=\{b, \neg b\}$, and $K_3'=\{c, \neg c\wedge d\}$.
	\item  $K_1''=\{\neg a, a\vee b, \neg b\}$, and $K_2''=\{\neg c\wedge d, \neg d\wedge e\wedge f\}$.
\end{itemize}
Note that $K_3'=\{c, \neg c\wedge d\}$ and $K_3''=\{\neg e, \neg d\wedge e\wedge f\}$ can not form a partial MUS-decomposition due to the violation of the condition (2) in Definition \ref{def:ccpartition}.
This also shows that condition (3) alone  can not guarantee to satisfy the condition 2 in the  definition.
\end{example}

\begin{figure}[!h]
\begin{center}
\includegraphics[width=8cm]{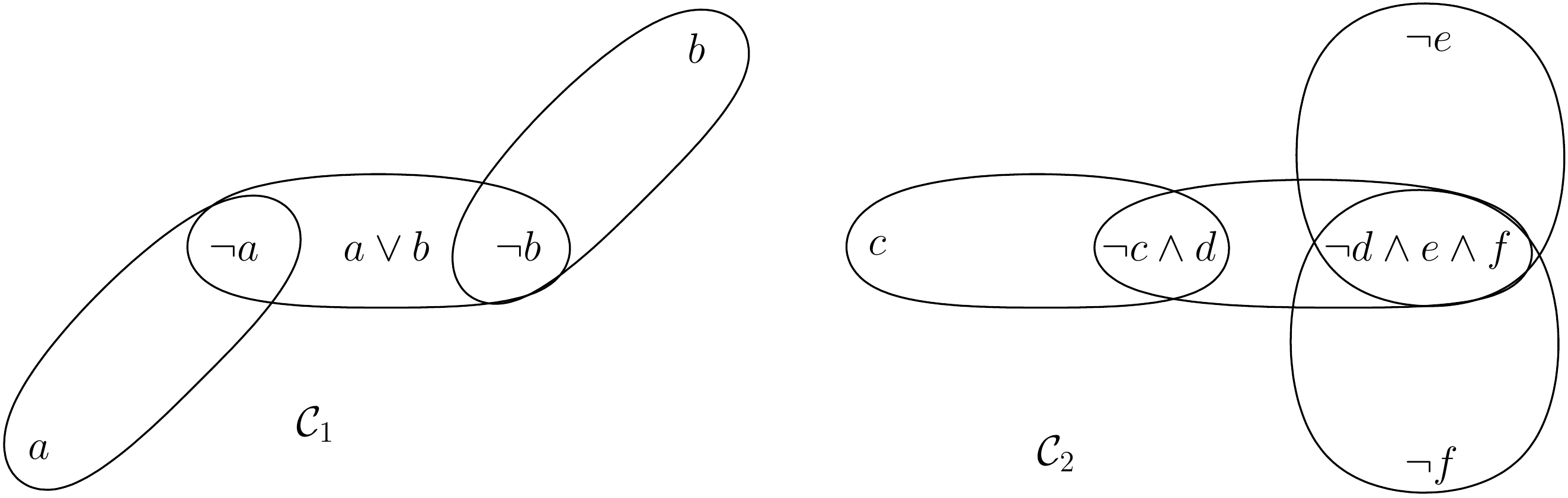} 
\caption{Connected components of $K$}
\label{fig:ubimuslb}
\end{center}
\end{figure}

\begin{definition}[Maximal  partial MUS-decomposition] \label{def:max-pMUSd}
 A partial MUS-decomposition  $T\in pMUSd(K)$  is called maximal if 
$|T| = \mu_{\cal D}(K),$ where $\mu_{\cal D}(K)$ is defined by
  $$\mu_{\cal D}(K)=\arg\max_{T'\in pMUSd(K)} |T'|.$$
Moreover, $\mu_{\cal D}(K)$ is called the distribution index of $K$.
\end{definition} 

That is, the maximal partial MUS-decomposition has the largest cardinality among all partial MUS-decompositions.
And the distribution index is the cardinality of maximal MUS-decompositions.

\begin{example}(Example \ref{ex:icc3} contd.)\label{ex:icc3-max}
Among maximal partial MUS-decompositions is $\{K'_1, K'_2, K'_3\}$.
Note that ${\cal C}_2$ contains highly connected formulae that cannot be separated into a partial MUS-decomposition of size larger than 2. 
\end{example}

Although a (maximal) partial MUS-decomposition can be formed by any subsets of $K$, the next proposition indicates that only $MUSes(K)$ are needed to obtain a (maximal) partial MUS-decomposition.

\begin{lemma}\label{lem:pMUSd}
Let $K$ be an inconsistent KB. There exist $\mu_{\cal D}(K)$ distinct MUSes $M_1,\ldots{}, M_{\mu_{\cal D}(K)}$ such that 
$\{M_i \mid 1\leq i \leq \mu_{\cal D}(K)\}$ is a maximal partial MUS-decomposition of $K$.
\end{lemma}

\begin{proof} 
Suppose $\{K_1,\cdots,K_{\mu_{\cal D}(K)}\}$ is a maximal MUS-decomposition of $K$.
Let $M_i \in MUSes(K_i)$ for $1\leq i\leq \mu_{\cal D}(K)$, then it is easy to verify that $\{{M}_1, \cdots, {M}_{\mu_{\cal D}}(K)\}$ is a partial MUS-decomposition whose cardinality is the distribution index of $K$, so it is a maximal pMUSd.
\end{proof}

That is, each element of a maximal partial MUS-decomposition can be some minimal unsatisfiable subsets of $K$, as $\{K'_1, K'_2, K'_3\}$ in Example \ref{ex:icc3}. 
Moreover, the following proposition tells that we can have another special format of maximal MUS-decomposition.

\begin{proposition}[Distributable MUS-decomposition]\label{lem:pMUSd1}
Let $K$ be an inconsistent KB.
There exist $\mu_{\cal D}(K)$ distinct  ${\cal M}_i\subseteq MUSes(K)$ for $1\leq i \leq\mu_{\cal D}(K)$, 
such that   $\{\bigcup_{M\in {\cal M}_i} M\mid 1\leq i \leq \mu_{\cal D}(K)\}$ is a maximal partial MUS-decomposition of $K$ and ${\cal M}_i$ is maximal w.r.t. set inclusion.
We call such a maximal partial MUS-decomposition a {distributable MUS-decomposition}.
\end{proposition}
\begin{proof} {By Lemma \ref{lem:pMUSd}, take a maximal MUS-decomposition of the form $\{M_i \in MUSes(K)\mid 1\leq i \leq \mu_{\cal D}(K)\}$. Denote ${\cal C}_i$ the connected component of ${\cal G}_{{\cal MUS}}(K)$ such that $M_i\in {\cal C}_i$.
Now consider ${\cal M}_i\subseteq{\cal C}_i$  such that $\{\bigcup_{M\in {\cal M}_i} M\mid 1\leq i \leq \mu_{\cal D}(K)\}$ is still a partial MUS-decomposition of $K$.
Such ${\cal M}_i$ exists because we can take ${\cal M}_i=\{M_i\}$. $K$ is finite, so are ${\cal G}_{{\cal MUS}}(K)$ and ${\cal C}_i$.
Now taking ${\cal M}_i$ that is  maximal w.r.t. set-inclusion with such a property, the conclusion follows.}
\end{proof}


\begin{example}(Example \ref{ex:icc3} contd.)\label{ex:dMUSd}
$\{K'_1, K'_2, {\cal C}_2\}$ is a distributable MUS-decomposition,  but $\{K'_1, K'_2, K'_3\}$ is not because $K'_3\subset {\cal C}_2$. 
\end{example}

\begin{example} \label{ex:dMUSd-intro} Recall the example in Section \ref{sec:intro}:
$K=\{a_1, \neg a_1, a_1\vee \neg a_2, a_2, \neg a_2, \cdots, a_{n-1}, \neg a_{n-1}, a_{n-1}\vee \neg a_{n}, a_{n}, \neg a_{n}\}$.
The distributable MUS-decomposition of $K$ is  $\{a_i,\neg a_i\}$.
\end{example}

A distributable  MUS-decomposition defines  a way to separate a whole KB into maximal number of disjoint inconsistent components. The decomposed components, such as $\{a_i, \neg a_i\}$, can in turn be delivered to $n$ different experts to repair in parallel. In the case where resolving inconsistency is a serious and time-consuming decision, this can advance task time by a distributed manipulation of maximal experts. 
Indeed, the rational in  distribution MUS-decomposition  related to inconsistency resolving is given in the following proposition.

\begin{proposition}
Given an inconsistent base $K$ and $T=\{K_1, \cdots, K_n\}$ is a  distributable MUS-decomposition of $K$. Suppose $K'_i$ is a consistent base obtained by removing or weakening formulae in $K_i$. Then 
$K'=\bigcup_{1}^nK'_i$ is consistent.
\end{proposition}
That is, inconsistencies in each component can be resolved  separately and the merged KB afterwards is consistent.
However, note that $K'\cup R$ where $R=K\setminus\bigcup_{i}K_i$ is not necessarily consistent\footnote{Indeed, this is unavoidable by Proposition \ref{prop:icc1hs} if each expert only removes one formula from $K$.}. For  instance, in Example \ref{ex:icc3-max}, if we have $K'_1=\{\neg a\}$  and $K'_2=\{\neg b\}$ after expert verification, we still have inconsistency in $\{\neg a, a\vee b,\neg b\}$. In this case, we can drop $a\vee b$ because $\{\neg a,\neg b\}$ have been manually chosen by experts; Or for carefulness, we can retrigger the same process to resolve the rest inconsistencies. 
\subsection{Distribution-based Inconsistency Degree}

As we can see above that a distributable MUS-decomposition gives a reasonable disjoint partition of a KB. In this section, we study the distribution index which rises an interesting inconsistency measure with desired properties.

 \begin{definition}
 \label{def:MeasureMIP}
Let $K$ be a KB, the distribution-based inconsistency degree $ I_{{\cal D}}(K)$ is defined as: 
$$ I_{{\cal D}}(K) = \mu_{\cal D}(K).$$
 \end{definition}
 
Intuitively, $ I_{{\cal D}}(K)$ characterizes how many experts are demanded to repair inconsistencies in parallel.
The higher the value is, more labor force is required\footnote{Note that labor force for resolving an inconsistent component can vary for different inconsistent components because of for instance differences in their sizes. But we focus on the study of maximal number of components in the present paper.}.

\begin{example}(Example \ref{ex:dMUSd} contd.) Since $\{K'_1,K'_2,{\cal C}_2\}$ is a distributable MUS-decomposition, we have $I_{{\cal D}}(K)   = 3$.
\end{example}

Indeed, the so defined measure satisfies several important properties for an inconsistency measure. 
\begin{proposition} 
	$I_{{\cal D}}(K)$ satisfies Consistency, Monotony, Free formula independence,  MinInc, and Independent Decomposability.
\end{proposition}

\begin{proof}
	\textit{Consistency}: If $K$ is consistent, the partial MUS-decomposition set is empty, so $I_{{\cal D}}(K)=0$. \\
	 \textit{Monotony}: For any KB $K$ and $K'$, it is easy to see that a partial MUS-decomposition of $K$ is a partial MUS-decomposition of $K\cup K'$.  Therefore, $\mu_{\cal D}(K)\leq \mu_{\cal D}(K\cup K')$.\\
	\textit{Free formula independence}: It follows from the obvious fact that free formula do not effect the set of partial MUS-decompositions.\\
	\textit{MinInc}: For $M\in MUSes(K)$, clearly, the only partial decomposition of $M$ is $\{M\}$, so $I_{{\cal D}}(K)=1$.\\
	\textit{Independent Decomposability}: Let $K,K'$ two bases satisfying  $MUSes(K) \oplus MUSes(K') = MUSes(K\cup K')$ and  $unfree(K) ~\cap~ unfree(K') = \emptyset$. For any  partial  MUS-decompositions of $K$ and $K'$: $\mathcal{M}=\{{ M}_1,\cdots,{ M}_{\mu_{\cal D}(K)}\}$ and  $\mathcal{M'}=\{{ M}'_1,\cdots,{ M}'_{\mu_{\cal D}(K')}\}$, it is easy to see  $\mathcal{M}\cup \mathcal{M'}\in pMUSd(K\cup K')$.  Moreover,  ${\cal M}\cup \mathcal{M'}$  is of the  maximal cardinality in $pMUSd(K\cup K')$. Otherwise, by Lemma \ref{lem:pMUSd}, there are ${ M}''_j\in MUSes(K\cup K')$ that form a partial MUS-decomposition of $K\cup K'$: $\{{ M}''_1,\cdots,{ M}''_N\}$ with $N>\mu_{\cal D}(K)+\mu_{\cal D}(K')$. Since $MUSes(K) \oplus MUSes(K') = MUSes(K\cup K')$, we have either
 $M''_j\in MUSes(K)$ or  $M''_j\in MUSes(K')$ for all $j$. So at least  one of $K$ and $K'$  has a partial MUS-decomposition whose cardinality is stricter larger than its distribution index. A contradiction with the definition of distribution index.
So $\mu_{\cal D}(K \cup K') = \mu_{\cal D}(K) + \mu_{\cal D}(K' )$. Consequently, $I_{{\cal D}}$ satisfies independent decomposability property.
\end{proof}

Moreover, the distribution-based inconsistency measure  is a lower bound of inconsistency measures which satisfy monotony, independent Decomposability, and MinInc properties.

\begin{proposition}
Given an inconsistency measure $I$ that satisfies Monotony, Independent Decomposability, and MinInc,  we have
 $I(K) \ge \mu_{\cal D}(K)$.
\end{proposition}

\begin{proof} For any partial MUS-decomposition $\{K_1,\ldots,K_n\}$ of $K$, we have $\bigcup\limits_{1\leq i\leq n}K_i \subseteq K$.
So by monotony, $I(K) \ge I(K_1 \cup \ldots \cup K_n)$.
Moreover, since $I$ satisfies independent Decomposability, $I(K) \ge I(K_1) + \ldots + I(K_n)$.
Taking a maximal partial MUS-decomposition, one  can deduce that $I(K) \ge I(K_1) + \ldots + I(K_{\mu_{max(K)}})$.
By MinInc and monotony, $I(K_i) \ge 1$, so $I(K) \ge \mu_{\cal D}(K)$.
\end{proof}

\begin{example}(Example \ref{exam:enh-additivity} contd.) \label{ex:ids}
	For different measures based on MUSes, we have
	\begin{itemize}
	\item[--] $I_{{\cal D}}(K_1\cup K_2)=1$ and $I_{{\cal D}}(K_1\cup K_3)=2$;
	\item[--] $\delta_{hs}(K_1\cup K_2)=1$ and $\delta_{hs}(K_1\cup K_3)=2$;
	\item[--]  $I'_{M}(K_1\cup K_2)=1$ and $I'_{M}(K_1\cup K_3)=4$;
	\item[--] $I_{MI}(K_1\cup K_2)=2$ and $I_{MI}(K_1\cup K_3)=2$.
	\end{itemize}
	So all $I_{{\cal D}},\delta_{hs},$ and $I'_{M}$ give a conclusion that  $K_1\cup K_2$ is less inconsistent than $K_1\cup K_3$, which coincides with our intuition, but it is not the case of $I_{MI}$.
\end{example}

In Example \ref{ex:ids}, we have $I_{{\cal D}}$ and $\delta_{hs}$ of  the same value.
But it is not the general case as shown in the following example. 

\begin{example}(Example  \ref{ex:icc3} contd.)
	For the  connected component ${\cal C}_2$,
	 $\delta_{hs}({\cal C}_2) = 2$  while its distribution index  is 1.
\end{example}

However, the following Proposition gives a general relationship between $I_{{\cal D}}$ and $\delta_{hs}$.

\begin{proposition}\label{prop:icc1hs}
Let $K$ be a KB.
We have $$ I_{{\cal D}}(K)  \le \delta_{hs}(K).$$
\end{proposition}

\begin{proof}
	As $K$ can be partitioned into $\mu_{\cal D}(K)$ disjoint  components of minimal inconsistent subsets of $K$, a minimal hitting set of $K$ must contain at least  one formula from each component.
That is, $I_{{\cal D}}(K)=\mu_{\cal D}(K)  \le \delta_{hs}(K)$.
\end{proof}

\begin{example}(Example \ref{ex:dMUSd-intro} contd.) We have $ I_{{\cal D}}(K) = n$ and $\delta_{hs}(K)=n$.
But the former means that $K$ can be distributed to $n$ experts to resolve inconsistency in parallel and each expert only verifies two elements because of the distribution MUS-decomposition is $\{a_i,\neg a_i\}$; Whilst the latter means that each expert needs to  verify at least $n$ formulae to confirm an inconsistency resolving plan.  And different experts  have to do repetition work due to  overlapping  among different hitting sets. 
\end{example}
This example shows that the proposed MUS-decomposition gives a more competitive inconsistency handling methodology than the hitting set based approach albeit the occasionally equivalent value of the deduced inconsistency measures  $ I_{{\cal D}}(K)$ and $\delta_{hs}$.

\section{Computations of  $I_{{\cal D}}(K)$}\label{sec:computation}

In this section, we consider the computational issues of distribution-based inconsistency measure $I_{{\cal D}}(K)$
by generalizing the classical Set Packing problem, and then show two encodings of   $I_{{\cal D}}(K)$, which is aiming at practical algorithms for its solution. 

We first look at the following proposition which is a simple conclusion of  Lemma \ref{lem:pMUSd}.

\begin{proposition} \label{prop:mumus}
Let $K$ be a KB. $I_{\cal D}$ is the maximal  cardinality of $\mathcal{M}\subseteq  MUSes(K)$ satisfying 
	\begin{enumerate} 
	  	\item $MUSes(\cup_{M \in \mathcal{M}} M) = \mathcal{M}$.  
	\end{enumerate}
\end{proposition}

Proposition \ref{prop:mumus}  states that  
$I_{\cal D}(K)$ is the largest number of (pairwise disjoint) MUSes of $K$ such that their union  will not rise any new MUS, which gives a way to compute $I_{{\cal D}}(K)$.

Next we study this computation in the framework of Maximum Closed Set Packing (MCSP) defined in the following.

\subsection{Closed Set Packing}

The maximum set packing problem is one of the basic optimization problems (see, e.g.,~\cite{setPacking79}).
It is related to other well-known optimization problems, such as the maximum independent set and maximum clique problems~\cite{AroraLMSS98,AroraS98,BoppanaH92,FeigeGLSS96,Wigderson83}. 
We here introduce a variant of this problem, called the \textit{maximum closed set packing problem}. We show that this variant is NP-hard by providing a reduction   from the maximum set packing problem  which is NP-hard~\cite{Karp72}. 
 In this work, the maximum closed set packing problem is used to compute the distribution-based inconsistency measure. 
 
Let $U$ be universe and  $S$  be a family of subsets of $U$. 
\begin{definition}[Set Packing]
A {\em set packing} is a subset $P\subseteq S$ such that, 
for all $S_i, S_j\in P$ with $S_i\neq S_j$, $S_i\cap S_j =\emptyset$. 
\end{definition}

Our variant is obtained from the maximum set packing problem by further requiring that the union of selected subsets does not contain unselected subsets in $S$ as defined below.

\begin{definition}[Closed Set Packing]
A {\em closed set packing} is a set packing $P\subseteq S$ such that, 
for all $S_i\in S\setminus P$,  $S_i$ is not a subset of $\bigcup_{P_i\in P} P_i$. 
\end{definition}

The {\em maximum  (free) set packing problem}  consists in founding a (free)  set packing with maximum cardinality, written MSP (MCSP).  

\begin{theorem} \label{thm:MFSP-complexity}
	MCSP is NP-hard.
\end{theorem}

\begin{proof}
We construct a reduction from the maximum set packing problem to the maximum closed set packing problem.
Let $U$ be a universe,  $S=\{S_1,\ldots{}, S_n\}$ a family of subsets of $U$ and $e_1,\ldots{},e_n$ are $n$ distinct elements which do not belong to $U$.
Define $U'=U\cup \{e_1,\ldots{}, e_n\}$ and $S'=\{S_1\cup\{e_1\},\ldots{}, S_n\cup\{e_n\}\}$.
We have $P$ is a solution of the  maximum set  packing problem for $(U,S)$ if and only if   
$P'=\{S_i\cup\{e_i\} \mid S_i \in P\}$ is a solution of the  maximum closed set packing problem for $(U', S')$.
Since maximum set packing MSP is NP-hard, so is the MCSP.
\end{proof}


%


\subsection{Integer Linear Program Formulation of MCSP}

We here provide an encoding of the maximum closed set packing problem in linear integer programming. 
Let $U$ be a universe and $S$ a set of subsets of $U$. 
We associate a binary variable $X_{S_i}$ ($X_{S_i}\in\{0,1\}$) to each subset $S_i$  in $S$. 
We also associate a binary variable $Y_e$ to each element $e$ in $U$.\\

The first linear inequalities allow us to only consider the pairwise disjoint subsets in $S$:
\begin{equation}
\label{eq1}
\sum_{e\in S_i, S_i\in S} X_{S_i}\leq 1~~~~for~all~e\in U
\end{equation}

The  following inequalities allow us to have $X_{S_i}=1$ if and only if, for all $e\in S_i$, $Y_e=1$:
\begin{equation}
\label{eq2}
(\sum_{e\in S_i} Y_e)-C_i*X_{S_i}\geq 0~~~~for~all~S_i\in S
\end{equation}
\begin{equation}
\label{eq3}
(\sum_{e\in S_i} Y_e)-X_{S_i}\leq C_i-1~~~~for~all~S_i\in S
\end{equation}
where, for all $S_i\in S$, $C_i=|S_i|$. Indeed, If $X_{S_i}=1$ then, using inequality~(\ref{eq2}), we have, for all $e\in S_i$, $Y_e=1$.
Otherwise, we have $X_{S_i}=0$ and, using inequality~(\ref{eq3}), there exists $e\in S_i$ such that $Y_e=0$.\\

Finally, the objective function is defined as follows: 
\begin{equation}
\label{eq4}
\max~\sum_{S_i\in S} X_{S_i}
\end{equation}

\begin{proposition}
The linear inequalities in $(\ref{eq1})$, $(\ref{eq2})$ and $(\ref{eq3})$ with the objective function $(\ref{eq4})$ is a correct encoding 
of MCSP.
\end{proposition}

\begin{proof}
Let $P$ be a subset of $S$ that corresponds to a solution of the linear integer program. 
Using the inequalities in (\ref{eq1}), we have, for all $S_i,S_j\in P$ with $S_i\neq S_j$, $S_i\cap S_j=\emptyset$.
Thus, $P$ corresponds to a set packing. 
Using the inequalities $(\ref{eq2})$ and $(\ref{eq3})$, we have, for all $S_i\in S$, $X_{S_i}=1$ if and only if, for all $e\in S_i$, $Y_e=1$.
Hence, for all $S_i\in S\setminus P$, there exists $e\in S_i$ such that $Y_e=0$, so $S_i$ is not a subset of $\bigcup_{P_i\in P} P_i$.
Therefore, $P$ is a closed set packing. Finally, from maximizing the objective function in $(\ref{eq4})$, we deduce that $P$ is a solution  of the maximum closed set packing for $(U,S)$.
\end{proof}


\subsection{MinCostSAT Formulation of MCSP  }

In this section, we describe our encoding of the maximum closed set packing problem as a MinCostSAT instance \cite{MiyazakiIK96}. 

\begin{definition}[MinCostSAT]
Let $\Phi$ be a CNF formula and $f$ a cost function that associates a non-negative cost to each variable in $Var(\Phi)$.
The MinCostSAT problem is the problem of finding a model for $\Phi$ that minimizes the objective function:
$$\sum_{p\in {Var(\Phi)}} f(p )$$
\end{definition}

Let $U$ be a universe and $S$ a set of subsets of $U$. 
We associate a boolean variable $X_{S_i}$  (resp. $Y_e$) to each $S_i\in S$ (resp. $e\in U$).
The inequalities in $(\ref{eq1})$ in our previous integer linear program correspond to  instances of the  AtMostOne constraint
which is a special case of the  well-known cardinality constraint.  
Several efficient encodings of the cardinality constraint to CNF have been proposed, most of them try 
to improve the efficiency of constraint propagation (e.g. \cite{BailleuxB03,Sinz05}).
We here consider the encoding using sequential counter~\cite{Sinz05,SilvaL07}. 
In this case, the inequality $\sum\limits_{e\in S_i, S_i\in S} X_{S_i}\leq 1$ is encoded as follows
(we fix $\sum\limits_{e\in S_i, S_i\in S} X_{S_i}=\sum\limits_{1\leq i\leq n} X_{S_i}$):

\begin{equation}
\label{atmost}
\begin{split}
(\neg X_{S_1} \vee p_1) \wedge (\neg X_{S_n} \vee \neg p_{n-1}) \\ \bigwedge_{1< i < n} ((\neg X_{S_i} \vee p_i) \wedge (\neg p_{i-1} \vee p_i) \wedge (\neg X_{S_i} \vee \neg p_{i-1}))
\end{split}
\end{equation}
where $p_i$ is a fresh boolean variable for all $1\leq i\leq n-1$.

Regarding to the inequalities in~(\ref{eq2}), it can be encoded by the following clauses:
\begin{equation}
\label{ceq2}
\bigwedge_{S_i\in S}\bigwedge_{e\in S_i}\neg X_{S_i}\vee Y_e
\end{equation}
Indeed, these clauses are equivalent to the following ones: 
\begin{equation*}
\bigwedge_{S_i\in S}X_{S_i}\rightarrow \bigwedge_{e\in S_i} Y_e
\end{equation*}
The inequalities in $(\ref{eq3})$ can be simply encoded as: 
\begin{equation}
\label{ceq3}
\bigwedge_{S_i\in S}(X_{S_i}\vee\bigvee_{e\in S_i}\neg Y_e)
\end{equation}

Contrary to MCSP,  the optimization process in MinCostSAT consists in minimizing the objective function. 
In order to encode MCSP as an MinCostSAT instance, we rename each variable $X_{S_i}$ with $\neg X_{S_i}'$  ($X_{S_i}'$ is a fresh boolean variable)
in (\ref{atmost}), (\ref{ceq2}) and (\ref{ceq3}),  for all $S_i\in S$. 
The MinCostSAT instance encoding the maximum closed set packing problem for $(U,S)$ is 
${\cal M}_{(U,S)}=(\Phi, f)$ where $\Phi$ is the CNF formula obtained from $(\ref{atmost})\wedge(\ref{ceq2})\wedge(\ref{ceq3})$ by the renaming described previously  
 and  $f$ is defined as follows:
\begin{itemize}
\item for all $S_i\in  S$, $f(X_{S_i}')=1$;  and 
\item for all $v\in Var(\Phi)\setminus\{X_{S_i}'\mid S_i\in S\}$, $f(v) = 0$.
\end{itemize}
Note that the optimization process in $\cal M$ consists in minimizing $\sum_{S_i\in S} X_{S_i}'$ and that corresponds 
to maximizing $\sum_{S_i\in S} X_{S_i}$.

\section{Experimental Results}
\label{sec:experiments}
In this section, we present a preliminary experimental evaluation of our proposed approach. 
All experiments was performed on a Xeon 3.2GHz (2 GB RAM) cluster.  


We conduced two kinds of experiments. The first one deals with instances coming from classical MUSes enumeration problem.
For this category we use two complementary state-of-the art MUSes enumeration
solvers and then we apply  our encoding into MCSP  to compute the values
of  $I_{{\cal D}}$. When enumerating all MUSes is infeasible we use {\tt eMUS} \cite{PrevitiM13} 
instead of  {\tt camus} \cite{LiffitonS08} to enumerate a subset of MUSes.  Indeed, {\tt eMUS}  is a real time solver that   
outperforms {\tt camus}  when we deal with partial MUSes enumeration. The instances where 
{\tt eMUS} is used are indicated with an asterisk.

In the second experiment, the instances are randomly generated. To represent a KB with $n$ formulae involving $m$ MUSes,  called {\tt mfsp\_m\_n},  we first generate randomly a family of sets $\{S_1, \ldots, S_m\}$ of positive integers from the interval $[1\dots n]$.  We suppose that each set $S_i$ of numbers represents a MUS. We randomly set the size of $S_i$. In our experiments, we consider  $1< |S_i|\leq 3$.



In Table \ref{tab:exp}, for each instance, we report  the number of MUSes ($\#mus$),  
the value of the inconsistency measure ($I_{{\cal D}}$) and the time ($time$ in seconds) needed
to compute $I_{{\cal D}}$ . To solve the encoded instances, we use ${\tt wpmaxsatz}$ Partial MaxSAT solver \cite{ArgelichDLMP06}.

As we can observe, the value $I_{{\cal D}}$  is much smaller than the number of MUSes. Furthermore, the computation time globally increases as $I_{{\cal D}}$ increases. Note that for instances whose  $I_{{\cal D}}$ value  is equal to 1, it means that they are strongly interconnected.  

 \begin{table}[htbp]
 \centering
 \begin{scriptsize}
 \begin{tabular}{l||l|l|l}
\hline
 Instance & $\#mus$   & $I_{{\cal D}}$ & $time(s)$ \\
 \hline

 C168\_FW\_UT\_851 &102 &1&1 \\
 C220\_FV\_RZ\_13 &6772 &1&5.4\\
 c880\_gr\_rcs\_w5.shufﬂed  &70&1&3.7\\
 rocket\_ext.b  &75&1&1\\
c7552-bug-gate-0$^{*}$  & 1000 & 1& 5.3\\
apex\_gr\_2pin\_w4.shuffled$^{*}$ & 1500 & 2 & 120.23 \\ 
wb\_conmax1.dimacs.filtered$^{*}$ & 20 & 2 & 0.9 \\
wb\_4m8s4.dimacs.filtered$^{*}$  &20 & 9& 1.44\\
\hline
mfsp\_50\_20         &  50          &   5 & 0.01        \\
mfsp\_100\_50       & 100           & 22 & 0.36       \\
mfsp\_120\_60       &120           & 15 & 1.49       \\
mfsp\_120\_80        & 120        & 20 & 13.78    \\
mfsp\_150\_60       & 150        & 11 & 1.50         \\
mfsp\_150\_100     & 150        & 22 & 127.57    \\
mfsp\_150\_150     & 150           & 35 & 347.98     \\
mfsp\_200\_50       & 200    &  11 & 4.79      \\
\hline
 \end{tabular}
 \end{scriptsize}
 
 \caption{Computation of $I_{{\cal D}}$ (real-world and random instances) \label{tab:exp}}
 \end{table}

%
%
%



\section{Conclusion}\label{sec:conclusion}

We  studied in this paper a new framework for characterizing inconsistency based on the proposed independent decomposability property and MUS-decomposition. Such defined inconsistency measures (I.e. $I'_M$ and $I_{\cal D}$) are shown with desired properties. The distributable MUS-decomposition allows to resolve inconsistencies in a parallel way, which is a rarely considered methodology for handling large knowledge bases with important informations. Complexity and practical algorithms are studied based on the advance of MUS enumeration.  We will study the lower bound complexity of the measure and explore applications of the proposed methodology in the future. 

%


\end{document}